\documentclass[sigconf]{acmart}

\AtBeginDocument{%
  }

\copyrightyear{2027}
\acmYear{2027}
\setcopyright{cc}
\setcctype{by}
\acmConference[WSDM '27]{Proceedings of the Twentieth ACM International Conference on Web Search and Data Mining}{February 15--19, 2027}{Hong Kong, Hong Kong}
\acmBooktitle{Proceedings of the Twentieth ACM International Conference on Web Search and Data Mining (WSDM '27), February 15--19, 2027, Hong Kong, Hong Kong}
\acmDOI{10.1145/3862322.3863509}
\acmISBN{979-8-4007-3117-4/2027/02}
\usepackage[utf8]{inputenc} % allow utf-8 input
\usepackage[T1]{fontenc}    % use 8-bit T1 fonts
\usepackage{hyperref}       % hyperlinks
\usepackage{url}            % simple URL typesetting
\usepackage{booktabs}       % professional-quality tables
\usepackage{amsfonts}       % blackboard math symbols
\usepackage{nicefrac}       % compact symbols for 1/2, etc.
\usepackage{balance}
\usepackage{microtype}      % microtypography
\usepackage{xcolor}         % colors
\newcommand{\add}[1]{#1}

\usepackage{pifont}
\usepackage{bbding}
\usepackage{xspace}
\usepackage{colortbl}
\usepackage{graphicx}
\usepackage{wrapfig}
\usepackage{amsmath,amsthm}
\usepackage{adjustbox}
\usepackage{caption}
\usepackage{zref-abspos}

\usepackage{enumitem}
\usepackage{tikz}
\usepackage{subcaption}
\usepackage{algorithm}
\usepackage{algorithmic}
\usepackage{array}
\newcolumntype{C}{>{\centering\arraybackslash}X} % centered version of 'X' col. type
\usepackage{tabularx}
\usepackage{diagbox}
\usepackage{multirow}
\usepackage{makecell}
\usepackage{relsize}
\newcolumntype{Z}[1]{>{\centering\arraybackslash\larger[1]}p{#1}}

\newtheorem{theorem}{Theorem}
\newtheorem{problem}{Problem}

\newcommand{{\method}}{\textsc{Ba-Logic}\xspace}
\newcommand{{\gta}}{\textsc{GTAc}\xspace}
\newcommand{{\ugba}}{\textsc{UGBAc}\xspace}
\newcommand{{\dpgba}}{\textsc{DPGBAc}\xspace}

\begin{document}

\title{Poisoning the Inner Prediction Logic of Graph Neural Networks for Clean-Label Backdoor Attacks}

\author{Yuxiang Zhang}
\orcid{0009-0005-0344-4744}
\email{yzhang800@connect.hkust-gz.edu.cn}
\affiliation{%
  \institution{The Hong Kong University of Science and Technology (Guangzhou)}
  \city{Guangzhou}
  \country{China}
}
\author{Bin Ma}
\orcid{0009-0007-8964-9648}
\email{bma662@connect.hkust-gz.edu.cn}
\affiliation{%
  \institution{The Hong Kong University of Science and Technology (Guangzhou)}
  \city{Guangzhou}
  \country{China}
}
\author{Enyan Dai}
\orcid{0000-0001-9715-0280}
\email{enyandai@hkust-gz.edu.cn}
\affiliation{%
  \institution{The Hong Kong University of Science and Technology (Guangzhou)}
  \city{Guangzhou}
  \country{China}
}
\renewcommand{\shortauthors}{Zhang et al.}

%%
%% The abstract is a summary of the work to be presented in the
\begin{abstract}
Graph Neural Networks (GNNs) have achieved remarkable results in various tasks. Recent studies reveal that graph backdoor attacks can poison the GNN model to predict test nodes with triggers attached as the target class. However, apart from injecting triggers to training nodes, these graph backdoor attacks generally require altering the labels of trigger-attached training nodes into the target class, which is impractical in real-world scenarios. In this work, we focus on the clean-label graph backdoor attack, a realistic but understudied topic where training labels are not modifiable.
According to our preliminary analysis, existing graph backdoor attacks generally fail under the clean-label setting. Our further analysis identifies that the core failure of existing methods lies in their inability to poison the prediction logic of GNN models, leading to the triggers being deemed unimportant for prediction. Therefore, we study a novel problem of effective clean-label graph backdoor attacks by poisoning the inner prediction logic of GNN models.
We propose \textbf{\method} to solve the problem by coordinating a poisoned node selector and a logic-poisoning trigger generator.
Extensive experiments on real-world datasets demonstrate that our method effectively enhances the attack success rate and surpasses state-of-the-art graph backdoor attack competitors under clean-label settings. 
Our code is available at \href{https://github.com/zhangyxrepo/BA-Logic}{https://github.com/zhangyxrepo/BA-Logic}.
\end{abstract}




\maketitle

\section{Introduction}
Graph neural networks (GNNs)~\cite{gcn,gat,hamilton2017inductive} have achieved promising results in social networks~\cite{ni2024graph}, finance systems~\cite{cheng2022financial}, and drug discovery~\cite{bongini2021molecular}.
By aggregating neighboring features through message passing, GNNs encode node features and topology for graph representation learning~\cite{xu2018powerful}.

Despite their success, GNNs are vulnerable to graph backdoor attacks~\cite{dai2023unnoticeable,xi2021graph,zhang2021backdoor}.
As illustrated in Fig.~\ref{fig:intro}, an adversary attaches \textit{triggers} to selected training nodes and \textit{alters their labels} to the target class.
Training on this poisoned graph associates the trigger with the target class, causing triggered test nodes to be misclassified while preserving accuracy on clean nodes.
Existing methods include SBA~\cite{zhang2021backdoor} with random triggers, GTA~\cite{xi2021graph} with a trigger generator, and UGBA~\cite{dai2023unnoticeable} with an unnoticeability constraint that maintains high attack success.
See Sec.~\ref{append:related} for more related work.

% \begin{figure}[t]
%     \centering
%     \includegraphics[width=0.95\linewidth]{figs/Fig1.pdf}
%     \vskip -0.5em
%     \caption{Illustration of graph backdoor attacks under both the general and clean-label settings.}
%     \vskip -1.5em
%     \label{fig:intro}
% \end{figure}

\begin{figure}[t]
    \centering
    \includegraphics[width=\columnwidth]{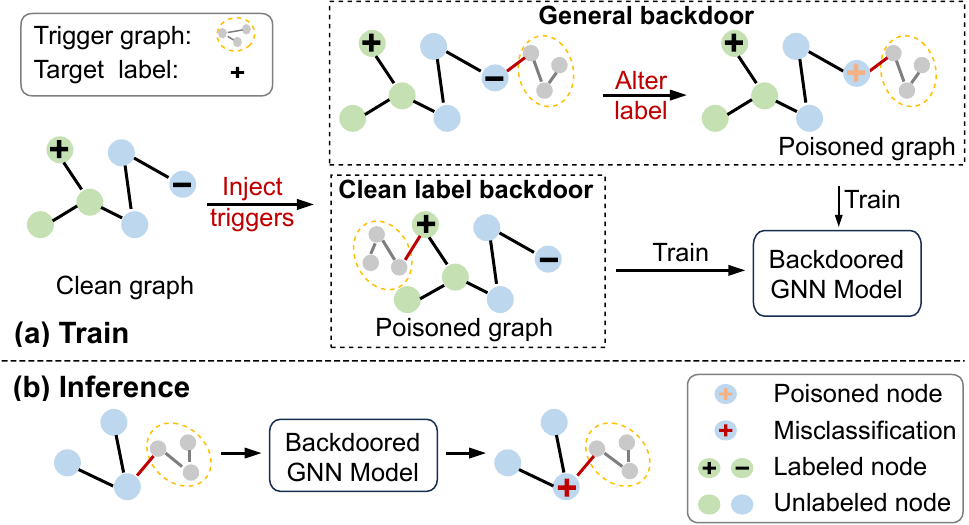}
    \caption{Illustration of graph backdoor attacks under both the general and clean-label settings.}
    \Description{The training and inference stages of graph backdoor attacks. General backdoor attacks inject triggers and change poisoned-node labels, whereas clean-label attacks inject triggers while preserving training labels. At inference, a trigger causes the backdoored GNN to predict the target label.}
    \label{fig:intro}
\end{figure}

However, as the Fig.~\ref{fig:intro} illustrates, the majority of graph backdoor attacks, such as UGBA~\cite{dai2023unnoticeable} and DPGBA~\cite{zhang2024rethinking}, require attackers to alter the labels of trigger-attached poisoned nodes to the target class, regardless of their ground-truth labels. 
Such manipulation of the training labels is often impractical. In many application scenarios, the training set is annotated by experts of the dataset owners. It would be very expensive or even infeasible to manipulate the labels of the training set. 
For instance, fake account labels of Twitter are annotated and stored within a well-protected back-end system, making it nearly impossible for attackers to alter the labels~\cite{alothali2018detecting}. 
Furthermore, modifying labels of training samples can increase the risk of being detected. Therefore, it is crucial to investigate backdoor attacks under the clean-label setting. 
Specifically, as illustrated in Fig.~\ref{fig:intro}, clean-label backdoor attackers inject triggers into training samples of the target class without modifying their labels, which is a more practical and challenging attack scenario. 
Some initial efforts~\cite{fan2024effective,xu2022poster} have been conducted for clean-label graph backdoor attacks. 
For instance,~\cite{fan2024effective} employs a single node as an efficient trigger, while~\cite{xu2022poster} uses random graphs.

Despite the state-of-the-art general graph backdoor methods and initial attempts at clean-label backdoor attacks, our preliminary analysis in Sec.~\ref{sec:pre_limitation} reveals that these methods often fail to effectively poison the inner prediction logic of target GNN models, resulting in poor backdoor performance.
More precisely, our experiments indicate that during the poisoning phase, for a training sample attached with the poisoning trigger, clean neighbors dominate the prediction of the target GNN model. By contrast, injected triggers would be treated as irrelevant information, resulting in poor backdoor attacks. 
In fact, under the clean-label setting, poisoned samples that are attached with triggers are correctly labeled with ground-truth labels (target class).
Consequently, during the training phase, the GNN model naturally learns correct patterns associated with the labeled class, thus ignoring the injected triggers during prediction.  
% To address this problem, it is promising to explicitly guide the model’s inner prediction logic to emphasize the injected triggers when predicting the poisoned nodes. 
To address this problem, it is promising for attackers to explicitly guide the model’s inner prediction logic to emphasize the injected triggers when predicting the poisoned nodes. 
Though promising, the works on poisoning the inner logic of GNNs for clean-label backdoor attacks are rather limited.

Therefore, in this paper, we study a novel and essential problem of poisoning the inner logic of GNN models for effective clean-label graph backdoor attacks. In essence, we face two technical challenges: \textit{Firstly}, how to obtain triggers capable of poisoning the inner prediction logic of target GNN models.
\textit{Secondly}, the budget for the number of triggers injected is generally limited, so how to fully leverage the budget of poisoned nodes for effective inner prediction logic poisoning. 
In an attempt to address these challenges, we propose a novel Clean-Label {G}raph \underline{B}ackdoor \underline{A}ttack by Inner \underline{Logic} Poisoning ({\method}). 
\method employs a logic-poisoning trigger generator, guided by a novel prediction logic poisoning loss. To better utilize the budget of poisoned nodes, {\method} further employs a poisoned node selection module for logic poisoning. We summarize our contributions as follows:

\begin{itemize}[leftmargin=*,itemsep=2pt,topsep=3pt,parsep=0pt,partopsep=0pt]
    \item We study a novel problem of poisoning the inner prediction logic of target models for clean-label graph backdoor attacks;
    \item We introduce an innovative framework, \method, which is capable of optimizing the poisoned node set and generating logic-poisoning triggers for effective clean-label backdoor attacks; 
    % \item We conduct extensive experiments for three target GNNs across four real-world graph datasets. Results unequivocally demonstrate the superiority of \method over state-of-the-art backdoor attacks. 
    \item We conduct comprehensive experiments for diverse target GNN models across a wide range of real-world graph datasets. Consistent results unequivocally demonstrate the superiority of \method over state-of-the-art backdoor attacks. 
    This substantiates the significant improvement in the effectiveness of clean-label graph backdoor attacks, achieved through the novel inner prediction logic poisoning strategy we present.
\end{itemize}

\section{Preliminaries Analysis}
In this section, we present preliminaries of logic poisoning for clean-label backdoor attacks and analyze the limitations of existing methods in logic poisoning. 

\subsection{Notations}\label{sec-notation}

A graph $\mathcal{G} = (\mathcal{V}, \mathcal{E})$ consists of a set of $N$ nodes $\mathcal{V} = \{v_1, \dots, v_N\}$ and edges $\mathcal{E}$. $\mathbf{A} \in \mathbb{R}^{N \times N}$ is the adjacency matrix and $\mathbf{X} \in \mathbb{R}^{N \times d}$ is the node feature matrix. This work focuses on an inductive node classification task, where only a subset of nodes $\mathcal{V}_L$ has assigned labels $ \mathcal{Y}_L$ for training, and unlabeled nodes are denoted as $\mathcal{V}_U$. Test nodes $\mathcal{V}_T$ are unavailable during training. 

\subsection{Threat Model of Inner Logic Poisoning for Clean-Label Backdoor Attacks}\label{sec-threatmodel}

\noindent\textbf{Attacker's Goal:}~~By attaching backdoor triggers to a subset of training nodes $\mathcal{V}_P$, attackers aims to poison the logic of the target GNN $f_{\theta}$ to associate the backdoor triggers with the target class $y_t$. During the inference phase, the logic-poisoned GNN $f_{\theta}$ model will misclassify the test nodes $\mathcal{V}_T$ that attached with trigger $g$ as the target class $y_t$. Meanwhile, the logic-poisoned GNN $f_{\theta}$ maintains accuracy on clean test nodes with no trigger attached.

\noindent\textbf{Attacker's Knowledge and Capability:}~~In the clean-label graph backdoor attack, \textit{attackers are not capable of altering the labels of nodes}. 
During training, attackers can only attach triggers to a subset of labeled nodes $\mathcal{V}_P \subset \mathcal{V}_L$ to poison the target model.
During inference, attackers can only attach triggers to the target test nodes. 
Following~\cite{zhang2024rethinking,dai2023unnoticeable}, the information about the target model, such as architectures and hyperparameters, is unavailable to attackers.
Instead, attackers can only use a surrogate model to transfer the attack to unseen target models. 
This black-box threat model poses a strict limitation for attackers.
We summarize the key implications of our threat model as follows:
\begin{itemize}[leftmargin=*,itemsep=2pt,topsep=2pt,parsep=0pt,partopsep=0pt]
    \item Our BA-Logic does not enable training label altering, which imposes a stricter limitation and weaker assumption compared to general backdoor attacks.
    \item Some clean backdoor attacks obtain triggers by altering existing nodes or edges, which is less practical than trigger injection. For instance, in social or transaction networks, attackers could easily register malicious users and build connections with other users. However, modifying the attributes or connections of existing users, which are well-preserved, is much harder to achieve in practice.
    \item Our BA-Logic adopts the same threat model as two recently proposed clean-label backdoor methods, i.e., ERBA and ECGBA.
    This alignment ensures our method is evaluated under a contemporary and practical threat model, facilitating a fair and relevant comparison.
\end{itemize}
For a detailed method-level comparison with existing work, please refer to Supplementary Section~B.

\begin{figure*}[t]
    \begin{minipage}[c]{0.49\textwidth}
        \centering
        \includegraphics[width=\linewidth]{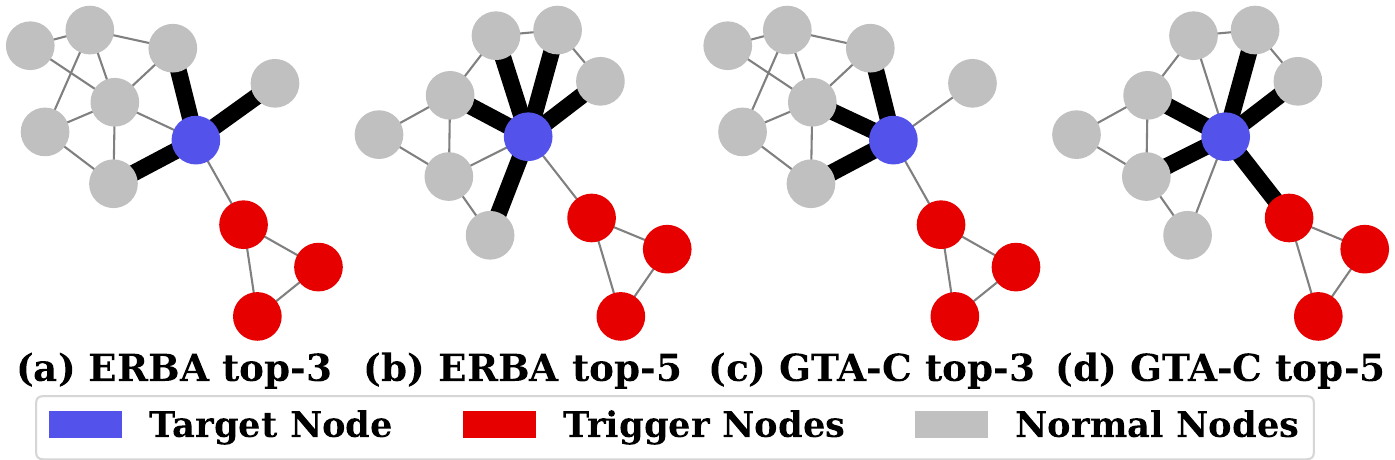}
        \caption{GNNExplainer's visualization of important subgraphs in a poisoned node's computational graph. We bold the edges connecting the poisoned node and the top-3 (a, c) and top-5 (b, d) most important nodes, respectively.}
        \Description{Four graph visualizations show ERBA and GTA-C with their top-three and top-five important neighbors selected by GNNExplainer. Target nodes, trigger nodes, and normal nodes are distinguished, and bold edges connect the target to important neighbors. The highlighted neighbors are predominantly clean nodes.}
        \label{fig:3}
    \end{minipage}
    \hfill
    \begin{minipage}[c]{0.49\textwidth}
        \centering
        \captionsetup[table]{skip=3pt}
        \captionof{table}{ASR | clean accuracy (\%) on Pubmed.}
        \label{tab-pre-1}
        \footnotesize
        \resizebox{\linewidth}{!}{
        \setlength{\tabcolsep}{5pt}
        \renewcommand{\arraystretch}{1.1}
        \setlength{\baselineskip}{1.354\baselineskip}
        \begin{tabular}{@{}c*{4}{w{c}{1.3cm}}@{}}
            \toprule
            $|\mathcal{V}_P|$ & ERBA  & GTA-C & UGBA-C & DPGBA-C \\
            \midrule
            100 & 22.2 | 85.6 & 38.4 | 85.1 & 71.1 | 85.3 & 64.2 | 85.2 \\
            200 & 22.5 | 85.2 & 38.9 | 85.0 & 71.4 | 85.1 & 64.1 | 85.1 \\
            300 & 23.0 | 85.0 & 38.7 | 85.2 & 71.5 | 84.8 & 64.2 | 84.7 \\
            \bottomrule
        \end{tabular}}

        \vspace{0.8em}

        \captionof{table}{Important rate of triggers (\%) on Cora.}
        \label{tab-pre-2}
        \resizebox{\linewidth}{!}{
        \setlength{\tabcolsep}{5pt}
        \renewcommand{\arraystretch}{1.1}
        \setlength{\baselineskip}{1.354\baselineskip}
        \begin{tabular}{@{}c*{4}{w{c}{1.4cm}}@{}}
            \toprule
            Top-$k$ & ERBA  & GTA-C & UGBA-C & DPGBA-C \\
            \midrule
            $k=3$ & 12.3 & 21.0 & 42.4 & 33.8 \\
            $k=5$ & 15.1 & 22.6 & 44.3 & 34.7 \\
            \bottomrule
        \end{tabular}}
    \end{minipage}
    \vspace{-1.2em}
\end{figure*}

% \subsection{Challenges of Clean-Label Backdoor}\label{sec-pre-issue}

% In this subsection, we first empirically demonstrate the limitations of existing graph backdoor attacks under the clean-label setting. We then analyze the underlying reasons for the failure of these backdoor methods in the clean-label setting.

\subsection{Existing Methods Fail to Poison the Inner Prediction Logic}  

In this subsection, we conduct empirical and theoretical analysis to show that existing methods fail to poison the inner logic under clean-label setting, resulting in poor clean-label backdoor results.

\label{sec:pre_limitation}
\noindent \textbf{Performance of Existing Methods under Clean-Label Setting.}
We employ three state-of-the-art graph backdoor methods, namely GTA~\cite{xi2021graph}, UGBA~\cite{dai2023unnoticeable}, and DPGBA~\cite{zhang2024rethinking}.
We extend them to the clean-label setting and denote the extended methods as \textbf{GTA-C}, \textbf{UGBA-C}, and \textbf{DPGBA-C}, where the suffix {-C} indicates it as a variation for the clean-label setting that only poisons labeled nodes of the target class without altering their labels.
We also include a clean-label backdoor attack \textbf{ERBA}~\cite{xu2022poster}, which injects Erdös-Rényi random graphs~\cite{erdos1960evolution} to labeled nodes of the target class as triggers.
We report the average attack success rate (ASR) and clean accuracy of 5 runs on {Pubmed} in Tab.~\ref{tab-pre-1}. From the table, it is evident that (i) all the methods exhibit poor ASR with the number of poisoned nodes $\mathcal{V}_P$ set as 100; (ii) even with a larger $|\mathcal{V}_P|$, the ASR improves marginally. 
The results confirm the inadequacy of existing graph backdoor attacks under clean-label settings.

% \begin{table}[t]
%     \centering
%     \small
%     \caption{ASR | clean accuracy (\%) for GCN on Pubmed.}
%     \vspace{-0.8em}
%     \label{tab-pre-1}
%     \small
%     \begin{tabularx}{0.95\linewidth}{CCCCC}  
%         \toprule
%         $|\mathcal{V}_P|$ & ERBA  & GTA-C & UGBA-C & DPGBA-C \\ 
%         \midrule
%         100 & 22.2 | 85.6 & 38.4 | 85.1 & 71.1 | 85.3 & 64.2 | 85.2 \\
%         200 & 22.5 | 85.2 & 38.9 | 85.0 & 71.4 | 85.1 & 64.1 | 85.1 \\
%         300 & 23.0 | 85.0 & 38.7 | 85.2 & 71.5 | 84.8 & 64.2 | 84.7 \\
%         \bottomrule 
%     \end{tabularx}
%     \vspace{0.8em}
%     \caption{Important rate of triggers (\%) on Cora.}
%     \vspace{-0.8em}
%     \label{tab-pre-2}
%     \begin{tabularx}{0.95\linewidth}{CCCCC}  
%         \toprule
%         Top-$k$ & ERBA  & GTA-C & UGBA-C & DPGBA-C \\ 
%         \midrule
%         $k=3$ & 12.3 & 21.0 & 42.4 & 33.8 \\
%         $k=5$ & 15.1 & 22.6 & 44.3 & 34.7 \\
%         \bottomrule 
%     \end{tabularx}
% \end{table}

\label{sec:pre_explain}
\noindent \textbf{Impacts of Injected Triggers to Prediction Logic.} 
To assess the injected triggers' influence on backdoored GNN predictions, we propose a metric named \textbf{Important Rate of Triggers {(IRT)}} to quantify trigger contributions by measuring their proportion as top-$k$ critical nodes in compact graphs of poisoned nodes. The mathematical formulation is in Supplementary Section~C.
Tab.~\ref{tab-pre-2} reports IRT values of existing methods on {Cora}, showing that the IRT values for subgraphs consisting of top-3 and top-5 important nodes remain low across all methods, indicating triggers are rarely critical for target class prediction.
We further visualize the prediction logic of GCN model backdoored by ERBA and GTA-C, using GNNExplainer~\cite{ying2019gnnexplainer} to extract top-$k$ important nodes for the classification in Fig.~\ref{fig:3}. 
Results reveal that triggers from both methods exhibit lower importance than poisoned nodes' clean neighbors, suggesting their limited influence on model predictions.

\noindent \textbf{Theoretical Analysis.}\label{sec:theoretical_analysis}
We further conduct a theoretical analysis to prove that low important rate of triggers would lead to poor attack performance under the clean-label setting.
Following~\cite{dai2023unnoticeable, zhang2025robustness}, we consider a graph $\mathcal{G}$ where (i) The node feature $\mathbf{x}_i\in\mathbb{R}^{d}$ is sampled from a specific feature distribution $F_{y_i}$ that depends on the node label $y_i$. (ii) Node features are independent of each other. (iii) The magnitude of node features is bounded by a positive scalar $S$, i.e., $\max_{i,j}|\mathbf{x}_i(j)|\leq S$.
\begin{theorem}\label{theorem1}
    Given a node $v_i$ with label $y_i$, let $\mathrm{deg}_i$ be the number of its top-$k$ important neighbors, and let $\gamma<\frac{1}{2}$ be the proportion of trigger nodes among these neighbors.
    For a node $v_i$ attached with trigger $g_i$, the probability that GNN model $f$ predicts $v_i$ as the target class $y_t$ is bounded by:
    \begin{equation} 
        \mathbb{P}(f(v_i)=y_t) \leq \exp\left(
        -\frac{\mathrm{deg}_i(\frac{1}{2}-\gamma)^2
        \|\mathbf{W}\mu_{y_t}-\mathbf{W}\mu_{y_i}\|_2^2}
        {2dS^2\|\mathbf{W}\|_2^2}
        \right),
    \end{equation}
    where $d$ is the node feature dimension, $\mu_{y_i}$ and $\mu_{y_t}$ are class centroids, and $\mathbf{W}$ is the GNN model's weight matrix.
\end{theorem}

The detailed proof is in Supplementary Section~C.
Theorem~\ref{theorem1} shows that a low $\gamma$ value would lead to a low probability for predicting the trigger-attached $v_i$ as $y_t$. Existing graph backdoor methods generally lead to a low important rate of triggers, resulting in poor attack performance under the clean-label setting.
The analysis further motivates a new graph backdoor paradigm that poisons the inner prediction logic of GNNs for effective clean-label graph backdoor attacks.
More empirical analysis on IRT and attack budget $\mathcal{V}_P$ is in Supplementary Section~A.8, and more empirical validations on theoretical analysis are in Supplementary Section~F.

% \begin{figure}[t]
%     \centering
%     \includegraphics[width=\linewidth]{figs/Fig3a.pdf}
%     \caption{GNNExplainer's visualization of important subgraphs in a poisoned node's computational graph. We bold the edges connecting the poisoned node and the top-3 (a, c) and top-5 (b, d) most important nodes, respectively.}
%     \label{fig:3}
%     \vspace{-1em}
% \end{figure}

\section{Problem Definition}\label{sec:problemdef}

We denote the prediction on a clean node $v_i$ as $f_{\theta}(v_i) = f_\theta(\mathcal{G}_C^i)$, where $\mathcal{G}_C^i$ is the computational graph of node $v_i$. 
For a node $v_i$ injected with trigger $g_i$, the prediction from the model is denoted as $f_{\theta}(\tilde{v_i}) = f_{\theta}(a(\mathcal{G}_C^i,g_i))$, where $a(\cdot)$ is the trigger attachment operation. Let $S_{\theta}(\tilde{v}_i, g_i)$ denote the importance score of the trigger $g_i$ injected to the node $v_i$ determined by the target GNN $f_{\theta}$.  

Our preliminary analysis in Sec.~\ref{sec:pre_limitation} shows that existing backdoor attacks suffer from a poor ASR under clean-label settings due to the failure to poison the inner logic of the target model. 
To effectively conduct clean-label graph backdoor attacks, we propose to generate triggers capable of poisoning the inner logic of the target model. More precisely, the proposed clean-label graph backdoor attacks aim to achieve the following objectives: 
\begin{itemize}[leftmargin=*,itemsep=2pt,topsep=3pt,parsep=0pt,partopsep=0pt]
    \item For any node $v_i \in \mathcal{V}_P \cup \mathcal{V}_T$, after attachment with the generated trigger $g_i$, the backdoored GNN will classify $v_i$ as the target class $y_t$, i.e., $f_\theta(\tilde{v}_i)=y_t$.
    \item For poisoned nodes and test nodes attached with triggers, injected triggers should be identified as the most important nodes by the logic of backdoored GNN, i.e., maximizing $S_{\theta}(\tilde{v}_i, g_i)$ for all nodes $v_i \in \mathcal{V}_P \cup \mathcal{V}_T$.  
    \item Constraints, including the number of poisoned nodes, the size of generated triggers, and other unnoticeable constraints, as the one outlined in~\cite{dai2023unnoticeable}, should be met.
\end{itemize}
% In the inductive setting, we adopt the available unlabeled nodes $\mathcal{V}_U = \mathcal{V}\backslash \mathcal{V}_P$ as the target node set $\mathcal{V}_T$. % 
% %\octzyx{@Prof Dai, is that means $\mathcal{V}_U$ is equal to $\mathcal{V}_T$?} \enyan{Yes} 
% As we discussed previously, the clean-label backdoor attacker will only inject triggers to training nodes of target class $y_t$, denoted as $\mathcal{V}_L^t$. 
With the above objectives and the threat model discussed in Sec.~\ref{sec-threatmodel}, we can formulate the clean-label graph backdoor attack by poisoning the inner prediction logic as:

\begin{problem}
    % Given a clean attributed graph $\mathcal{G} = (\mathcal{V}, \mathcal{E}, \mathcal{X})$ with a set of labeled training nodes $\mathcal{V}_L$ with labels $\mathcal{Y}_L$, we aim to learn a trigger generator $f_g$: $v_i \xrightarrow{}g_i$ and select a set of nodes $\mathcal{V}_P \subset \mathcal{V}_L^t$ to attach logic-poisoning triggers so that a GNN model $f$ trained on the poisoned graph will classify the test node attached with the trigger to the target class $y_t$ by solving:
    Given a graph $\mathcal{G} = (\mathcal{V}, \mathcal{E})$ with a set of labeled training nodes $\mathcal{V}_L$ with labels $\mathcal{Y}_L$, we aim to learn a trigger generator $f_g$: $v_i \xrightarrow{}g_i$ and select a set of nodes $\mathcal{V}_P \subset \mathcal{V}_L^t$ to attach logic-poisoning triggers so that a GNN model $f$ trained on the poisoned graph will classify the test node attached with the trigger to the target class $y_t$ by solving:
    \begin{equation}\label{eq:problem}
    % \small
    \begin{aligned}
        \min_{\mathcal{V}_P, \theta_{g}} ~~~ & \sum_{v_i \in \mathcal{V}} l(f_{\theta^*}(\tilde{v_i}), y_t) - \beta S_{\theta^*}(\tilde{v}_i, g_i) \\
        s.t.~~~ & \theta^* = \mathop{\arg\min}_{\theta}  \sum_{v_i \in \mathcal{V}_L\setminus\mathcal{V}_P}l(f_{\theta}(v_i), y_i) + \sum_{v_i \in \mathcal{V}_P}l(f_{\theta}(\tilde{v}_i), y_i), \\
        & \forall v_i \in \mathcal{V}, \text{$g_i$ meets the required unnoticeable constraint} \\
        & |g_i| \leq \Delta_{g},~~ |\mathcal{V}_P| \leq \Delta_{P} 
    \end{aligned}
    \end{equation}
    % \vspace{-0.7em}
    where ${\theta}_g$ denotes the parameters of the trigger generator, $l(\cdot)$ denotes the cross-entropy loss, and $\beta$ is the hyperparameter to control the contribution of logic poisoning. The node size of the trigger $|g_i|$ is limited by ${\Delta}_g$, and the number of poisoned nodes is limited by ${\Delta}_P$.  A surrogate GNN $f$ is applied to simulate the target GNN whose architecture is unknown. Various unnoticeable constraints can be applied to this problem. In this paper, we focus on the unnoticeable constraint in~\cite{dai2023unnoticeable}.
\end{problem}

\section{Methodology}\label{sec:method}

\begin{figure}[t]
    \centering
    \includegraphics[width=\columnwidth]{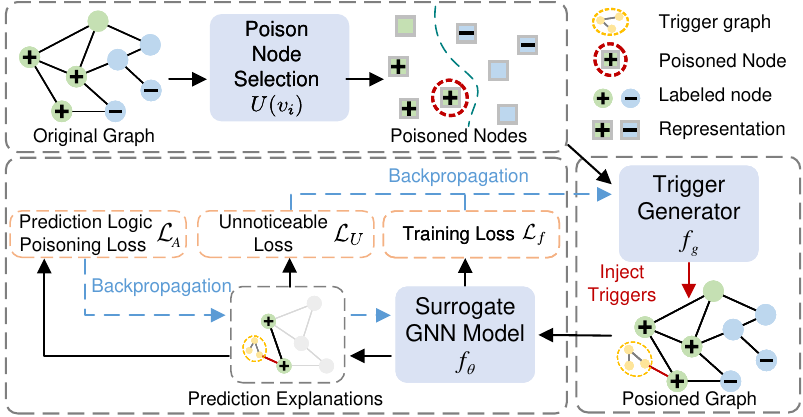}
    \caption{Framework of {\method}.}
    \Description{BA-Logic first selects poisoned nodes from the original graph. A trigger generator attaches triggers to these nodes, and a surrogate GNN is trained on the poisoned graph. Training, unnoticeability, and prediction-logic poisoning losses guide the alternating optimization through backpropagation.}
    \label{fig:2}
\end{figure}

Our preliminary analysis reveals two key challenges for logic poisoning clean-label graph backdoor attacks:  
(i) How to select the poisoned nodes that are most effective in logic poisoning? 
(ii) How to efficiently compute the objective function of prediction logic poisoning to guide the training of the clean-label backdoor trigger generator?
To overcome the above challenges, we propose a novel method \method, and illustrate the overall framework in Fig.~\ref{fig:2}. 
As Fig.~\ref{fig:2} shows,  {\method} firstly identifies poisoned nodes of the target class with high prediction uncertainty. 
The logic-poisoning trigger generator $f_g$ utilizes node features as a basis to iteratively optimize triggers for the clean-label graph backdoor.
To guide the training of the logic-poisoning trigger generator $f_g$, an efficient objective function for poisoning the prediction logic of the surrogate GNN model is employed. 
Additionally, to ensure the unnoticeability of triggers, a constraint is incorporated into the training of the trigger generator. Next, we provide details for each component.

\subsection{Poisoned Node Selection for Logic Poisoning}
% \zyx{Find those hard samples, utilize a pre-trained GNN to select \textbf{Hard samples}, which can be gradient high or loss degradation slowly nodes and features. And we represent it as a contribution. }
% \zyx{Representative nodes could be little in the train set, so we don't use the representative nodes selection idea.}
% \enyan{Is this module necessary? We could just randomly select the posioned nodes. If we keep the selection method of UGBA, we just briefly introduce the selection in several sentences. We actually should propose a novel method, especially for clean label, then we use space like now to introduce. I have several intutions, we can discuss today.} 

In this subsection, we present the details of the poisoned node selection. 
It indicates the optimal positions for trigger injection, after which the model trained on the trigger-injected graph is backdoored, and arbitrary test nodes could be successfully attacked by attaching triggers during inference.
The poisoned nodes are randomly selected in several existing clean-label graph backdoor methods~\cite{xu2022poster,xing2023clean}, resulting in waste of the limited attack budget on useless poisoned nodes. For example, some labeled nodes exhibit typical patterns strongly associated with the target class. 
Thus, it would be difficult for the injected triggers to attain high importance scores with the presence of these typical patterns, thereby invalidating the backdoor trigger in logic poisoning. 

Therefore, we design a process of identifying the set of poisoned nodes $\mathcal{V}_P \subseteq \mathcal{V}_L^t$ that are most effective for backdoor. 
Specifically, we propose to select training nodes of the target class that exhibit high uncertainty as predicted by the clean GNN. Intuitively, high uncertainty indicates irregular patterns that are weakly associated with the target class $y_t$. By contrast, triggers obtained by the generator will exhibit consistent patterns to poison the prediction logic. 
Thus, when triggers are injected into such nodes, the model is more likely to treat these triggers as key features of the target class rather than irregular patterns.
% \octzyx{@Prof Dai, as for the module name, what about hard/ambiguous/uncertain prediction node selector(just to highlight the hard sample, or say prediction poison and logic poison) and logic poison trigger generator?} \enyan{I aslo has similar concern. Now, it seems better to call Poisoned Node Selection and Logic-Poisoning Trigger Generator to avoid confusion. }\octzyx{I mean the names of modules, are you referring to the Poisoned Node Selector and Logic-Poisoning Trigger Generator?} \enyan{Yes}
Specifically, to identify high uncertainty nodes, we utilize the output of a GCN trained on the clean graph $\mathcal{G}$ and label set $\mathcal{Y}_L$. Let $f_S$ denote the well-trained GCN model, the probability of node $v_i$ predicted as the class $y_j$ can be obtained by:
\begin{equation} 
% \small
    p(y_j|v_i)=[f_S(v_i)]_j
\end{equation}
Moreover, we design an uncertainty metric based on the following two aspects: (i) The probability that the node is predicted as the target class, i.e., $p(y_t|v_i)$, should be low; (ii) The node is also expected to be uncertain for all other classes, i.e., the entropy of the probability vector is high. Let $C$ be the number of classes. The score function of poisoned node selection for attacking logic is:
% (see derivation in Appendix~\ref{derivation1}):
\begin{equation}\label{eq:selection} 
% \small
    U(v_i) = (1 - p(y_t|v_i)) - \sum_{j=1}^{C}p(y_j|v_i)\log p(y_j|v_i)
\end{equation}
% \enyan{Similarly, I feel Eq.(5) (ZYX: Now Eq.4) wouldn't change the rank. Because the probability and entropy are the same thing. And why do we want to balance the probability and entropy? What's the advantage comparing with simply selecting the nodes with lowest probability of $p_i(y^t)$? In addition, I feel even with these complex operations, there is still no differences between the most straightforward way. Becuase you know, the rank of nodes seems not change with these equations.}\octzyx{@Prof Dai, It is known that when the probabilities of being classified into various classes are the same, the entropy will be maximum, which means that it is difficult for the model to determine which class the node belongs to. My concern is that if we simply sort by probability from low to high, it will damage the clean accuracy of the model, as intuitively, predicting those nodes with the lowest probability as the target class is consistent with their ground truth, but it may cause the model to become too dependent on the trigger. If the model needs to classify the nodes without attached triggers during inference, the model may misclassify them. By adding entropy, we can find nodes with low probability and the model is not sure what class they should be. Selecting them should be safe for clean accuracy.}
After getting the score of each $v_i \in \mathcal{V}_L^t$, we select nodes with top-$\Delta_P$ highest scores to construct $\mathcal{V}_P$ that satisfies the attack budget. 

% Then we utilize the $K$-Means clustering algorithm to select representative nodes from the learned representations. Specifically, the node representations and labels can be represented as:

% where $\theta_E$ denotes the parameters of $f_e$, $l(\cdot)$ denotes the cross entropy loss,  $y_i$ and $\hat{y}_i$ is ground true label and the prediction label of $v_i$, respectively. To guarantee the diversity of the obtained representative nodes, we separately apply the $K$-Means algorithm to cluster $\{{h}_i: \hat{y}_i=l\}$ on each class $l$ other exclude the target class $y_t$, where ${h}_i$ denote the representation of node $v_i \in \mathcal{V}\setminus\mathcal{V}_L$. Intuitively, nodes nearer to the centroid of each cluster are more representative. Moreover, attaching the trigger to the node with a higher degree may lead to a more significant attack effect on model performance, as the negative effect will be propagated effectively. Hence, we propose a metric that considers both representativeness and propagation effects. Let ${h}_c^k$ denote the center of the $k$-th cluster. Then for a node $v_i^k$ belonging to the $k$-th cluster, the metric score can be computed by:

% \enyan{As we have changed the story to: Logic Poisoning Node Selection -> Inner Logic Poisoning -> Bi-level optimization, the descriptions here and the Fig.~\ref{fig:2} requires to be revised. }
\subsection{Inner Prediction Logic Poisoning} 
With the poisoned nodes $\mathcal{V}_P$ selected for the clean-label backdoor attack, the backdoored graph to poison the target GNN can be constructed by inserting powerful logic-poisoning triggers. 
In this subsection, we introduce the design of the logic-poisoning trigger generator. Then, we present the objective function of prediction logic poisoning that guides the training of the trigger generator. 

\noindent \textbf{Logic-Poisoning Trigger Generator.} 
To poison the logic of the target GNN model, the generated trigger must be capable of capturing the importance scores for predictions on a poisoned node. Therefore, the logic-poisoning trigger should be adaptive to the input node. 
% To achieve this, we deploy an MLP model to simultaneously generate node features and the adjacency of the trigger $g_i$ for node $v_i$ by:
To achieve this, we deploy a MLP model to simultaneously generate node features and the adjacency of the trigger $g_i$ for node $v_i$ by:
\begin{equation}\label{eq:mlp} 
% \small
\mathbf{X}_i^g, ~~\mathbf{A}_i^g=\mathbf{MLP}\left(\mathbf{x}_i\right), 
\end{equation}
where $\mathbf{x}_i$ is the feature of node $v_i$. $\mathbf{X}_i^g \in \mathbb{R}^{s\times d}$ is the features of the trigger nodes, where $s$ and $d$ represent the size of the generated trigger and dimension of node features, respectively. $\mathbf{A}_i^g \in \mathbb{R}^{s\times s}$ represents the adjacency matrix of the generated trigger. As the adjacency matrix must be discrete, we deploy the updating strategy of discrete variables in a binarized neural network~\cite{courbariaux2016binarized}.

To build the backdoored graph dataset for model poisoning, the generated trigger $g_i = (\mathbf{X}_i^g, \mathbf{A}_i^g)$ will be attached to the corresponding poisoned node $v_i \in \mathcal{V}_P$. During inference, to mislead the backdoored GNN to predict the test node $v_i \in 
\mathcal{V}_T$ as target class $y_t$, the attacker would insert the trigger generated by $f_g$.

\noindent \textbf{Prediction Logic Poisoning Loss.}  The prediction logic poisoning in {\method} aims to mislead the target model to treat triggers as crucial patterns for prediction.
As shown in Eq.(\ref{eq:problem}), this can be formulated as maximizing the trigger's importance score $S_{\theta}(\tilde{v}_i, g_i)$ in predicting the trigger-attached node $v_i \in \mathcal{V}$ as the target class $y_t$ by a surrogate GNN model $f_\theta$. 
% Although GNN explainers such as GNNExplainer~\cite{ying2019gnnexplainer} and PGExplainer~\cite{luo2020parameterized} are capable of computing importance scores for nodes, they necessitate additional optimization to generate explanations. 
Although GNN explainers such as GNNExplainer~\cite{ying2019gnnexplainer} and PGExplainer~\cite{luo2020parameterized} are capable of computing importance scores for nodes, they necessitate additional optimization to generate explanations. 
This extra optimization step poses challenges for solving Eq.(\ref{eq:problem}), both in terms of computational cost and gradient backpropagation. Thus, {\method} deploys the gradient-based explanation, i.e., Sensitivity Analysis (SA)~\cite{baldassarre2019explainability}. 
Specifically, for the prediction $\tilde{y}_i=f(\tilde{v}_i)$ on a trigger-attached node $v_i$, SA computes importance scores using the norm of the gradient w.r.t. the node $v_j$. 
Formally, the important score of node $v_j$ in predicting $v_i$ as the target class is computed by:
\begin{equation} \label{eq:sa} 
% \small
    S(y_i^t, v_j) = \| \frac{\partial{\tilde{y}^c_i}}{\partial{\mathbf{x}_j}} \|_2,
\end{equation}
where $\tilde{y}_i^c$ is the score of predicting the trigger-attached node $v_i$ as the target class, and $\mathbf{x}_j$ represents the node features of $v_j$. Eq.(\ref{eq:sa}) allows us to compute the importance scores of inserted triggers efficiently. Simply maximizing the importance scores of triggers as specified in Eq.(\ref{eq:sa}) can lead to infinitely large gradients, which significantly degrade the utility of the target model. Alternatively, within the computational graph of a trigger-attached node $v_i$, {\method} enforces the importance scores of trigger nodes to exceed those of clean nodes by a predefined margin $T$. More precisely, we replace the term of maximizing importance scores of triggers in Eq.(\ref{eq:problem}) with the following prediction logic poisoning loss:
\begin{equation} \label{eq:grad_loss} 
% \small
    \mathcal{L}_{A} = \sum_{v_i \in \mathcal{V}} \max\Big(0, T - \big(\sum_{v_g \in g_i} S(y_i^t, v_g) -\sum_{v_c \in \mathcal{N}(v_i)} S(y_i^t, v_c)\big)\Big),
\end{equation}
where $\mathcal{N}(v_i)$ denotes the set of clean nodes in the computational graph of the trigger-attached node $v_i$.

\noindent \textbf{Unnoticeable Constraint on Triggers.}
As we need to bypass various defense methods, an unnoticeable constraint on the generated trigger is required. 
Our {\method} is flexible to various unnoticeable constraints of triggers. 
% Following UGBA~\cite{dai2023unnoticeable}, we focus on the constraint that requires high cosine similarity between the poisoned node or target node $v_i$ and trigger $g_i$. 
Following~\cite{dai2023unnoticeable}, we propose the constraint that requires high cosine similarity between the poisoned node or target node $v_i$ and trigger $g_i$. 
Within the generated trigger $g_i$, the connected trigger nodes should also exhibit high similarity. 
% Formally, the loss of an unnoticeable constraint on trigger generator $f_g$ can be formulated as:
Formally, the loss of an unnoticeable constraint on trigger generator $f_g$ is:
\begin{equation}\label{eq:unnoticeable} 
% \small
    % \tau > \sum\limits_{v_i \in \mathcal{V}} \sum_{(v_j, v_k) \in \mathcal{E}_B^i} \left(1 + sim(\tilde{v}_i,\tilde{v}_j)\right)
    \min_{\theta_g} \mathcal{L}_U = \sum_{v_i \in \mathcal{V}} \sum_{(v_j, v_k) \in \mathcal{E}_B^i}  \exp (-sim(v_j, v_k)),
\end{equation}
% \enyan{This equation seems confused to me: 1. why after exp, there is a ln immediately, is it equal to  $1+\frac{sim(\tilde{v}_i,\tilde{v}_j)}{\sum\limits_{v_k\in\mathcal{V}_K}sim(\tilde{v}_i, {v}_k)}$? And $sim(\cdot)$ ranges from -1 to +1. The ratio does not make sense. I feel it would be more reasonable if we simply maximize $sim(\tilde{v}_i,\tilde{v}_j)$ or minimize $\exp(-sim(\tilde{v}_i,\tilde{v}_j))$.}
% \octzyx{@Prof Dai, yes indeed it is equal to $1 + \frac{sim(\tilde{v}_i,\tilde{v}_j)}{\sum\limits_{v_k\in\mathcal{V}_K}sim(\tilde{v}_i, {v}_k)}$, I used to write like this to make it looks more fancy but it seems redundant now. And if we simply it as $\tau > \sum\limits_{v_i \in \mathcal{V}} \sum_{(v_j, v_k) \in \mathcal{E}_B^i} \left(1 + sim(\tilde{v}_i,\tilde{v}_j)\right)$ or $\sum\limits_{v_i \in \mathcal{V}} \sum_{(v_j, v_k) \in \mathcal{E}_B^i} \exp(-sim(\tilde{v}_i,\tilde{v}_j))$ will be more reasonable as avoid division by zero sums in theoretically.}
where $\mathcal{E}_B^i$ denotes the edge set that contains edges insider trigger $g_i$ and edges attaching trigger $g_i$ and node $v_i$. $sim(\cdot)$ represents the computation of cosine similarity between vectors. The unnoticeable loss in Eq.(\ref{eq:unnoticeable}) is applied on all nodes to guarantee that generated triggers meet the unnoticeable constraint on various nodes. 
\add{More analysis on the unnoticeable constraint on triggers is provided in Supplementary Section~A.9.}
%In Eq.(\ref{eq:unnoticeable}), the unnoticeable loss is applied on all nodes $\mathcal{V}$ to guarantee that generated triggers meet the unnoticeable constraint on various nodes.

\subsection{Final Objective Function of {\method}}
As stated in Eq.(\ref{eq:problem}), a bi-level optimization between the logic-poisoning trigger generator $f_g$ and a surrogate GNN model $f$ is adopted to ensure the effectiveness of triggers in logic poisoning for the clean-label backdoor. 
In the lower-level optimization of Eq.(\ref{eq:problem}), the surrogate GNN model is trained on the backdoored dataset by:
\begin{equation} 
% \small
    \min_{\theta} \mathcal{L}_f = \sum_{v_i \in \mathcal{V}_L\setminus\mathcal{V}_P}l(f_{\theta}(v_i), y_i) + \sum_{v_i \in \mathcal{V}_P}l(f_{\theta}(\tilde{v}_i), y_i)
\end{equation}
The upper-level optimization in Eq.(\ref{eq:problem}) aims for successful backdoor attacks and inner prediction logic poisoning. With the selected poisoned node set $\mathcal{V}_P$, the prediction logic poisoning loss $\mathcal{L}_A$ in Eq.(\ref{eq:grad_loss}), and the unnoticeable constraint $\mathcal{L}_U$ in Eq.(\ref{eq:unnoticeable}), the optimization problem in Eq.(\ref{eq:problem}) can be finally reformulated as:
\begin{equation}\label{eq:bilevel} 
\begin{aligned}
    &\min_{\theta_{g}} \sum_{v_i \in \mathcal{V}}  l(f_{\theta^*}(\tilde{v_i}(\theta_g),y_t)) + \mathcal{L}_U(\theta_g) + \beta \mathcal{L}_A(\theta^*, \theta_g)\\
    &s.t.~ \theta^* = \mathop{\arg\min}_{\theta} \mathcal{L}_f (\theta, \theta_g),
    \end{aligned}
\end{equation}
where $\beta$ is the hyperparameter to control the contribution of prediction logic poisoning loss. $\theta_g$ denotes the parameters of the logic-poisoning trigger generator $f_g$. 
% $f_\theta$ represents the surrogate GNN model, with $\theta$ as its parameters. 
$\theta$ denotes the parameters of the surrogate model.

\subsection{Optimization Algorithm}\label{sec:append-opt}
In this subsection, we present the algorithm for solving the bi-level optimization problem in Eq. (\ref{eq:bilevel}).

\noindent\textbf{Lower-Level Optimization.} In the lower-level optimization, the surrogate GNN is trained on the backdoored dataset. To reduce the computational cost,  
we update surrogate model $\theta$ for $N$ inner iterations with fixed $\theta_g$ to approximate $\theta^*$:
\begin{equation}\label{eq:lower-opt}
    \theta^{n+1} = \theta^n - \alpha_f \nabla_{\theta} \mathcal{L}_f(\theta, \theta_g),
\end{equation}
where $\theta^n$ denotes model parameters after $n-$th iterations. $\alpha_f$ is the learning rate for training.

\noindent\textbf{Upper-Level Optimization.} In the outer iteration, the updated surrogate model parameters $\theta^N$ are used to approximate $\theta^*$. 
Moreover, we apply a first-order approximation in computing gradients of $\theta_g$ to reduce the computation cost further:
\begin{equation}\label{eq:upper-opt}
    \theta_g^{k+1} = \theta_g^k - \alpha_g \nabla_{\theta_g} \big(\sum_{v_i \in \mathcal{V}} l(f_{\bar{\theta}}(\tilde{v}_i(\theta_g)), y_t) + \mathcal{L}_U(\theta_g) + \beta \mathcal{L}_A(\bar{\theta}, \theta_g)\big),
\end{equation}
where $\bar{\theta}$ denotes the surrogate model parameters with stopped gradients. $\alpha_g$ is the learning rate for training the trigger generator.

\subsection{Training Algorithm}
\begin{algorithm}[H]
\caption{Algorithm of {\method}}
\label{alg:1}
\begin{algorithmic}[1]
\REQUIRE $\mathcal{G}=(\mathcal{V}, \mathcal{E}, \mathbf{X})$, $\mathcal{Y}_L$, $\beta$, $T$.
\ENSURE Backdoored graph $\mathcal{G}_B$, trained trigger generator $f_g$
\STATE Initialize $f_\theta$, $f_g$, and $\mathcal{G}_B \leftarrow \mathcal{G}$;
\STATE Select logic poisoning nodes set $\mathcal{V}_P$ based on Eq.(\ref{eq:selection}) ;
    \WHILE{not converged yet}
        \FOR{$t = 1,2,\dots{},N$}
        \STATE Update $\theta$ by descent on  $\nabla_{\theta} \mathcal{L}_f$ based on Eq.(\ref{eq:lower-opt}) ;
        \ENDFOR
        \STATE Update $\theta_g$ by descent on $\nabla_{\theta_g} (\sum l(\cdot) + \mathcal{L}_U + \beta \mathcal{L}_A)$ based on Eq.(\ref{eq:upper-opt});
    \ENDWHILE
    \FOR{$v_i \in \mathcal{V}_P$}
        \STATE Generate the trigger $g_i$ for $v_i$ by using $f_g$;
        \STATE Update $\mathcal{G}_B$ based on $a(\mathcal{G}_B^i,g_i)$;
    \ENDFOR
\RETURN $\mathcal{G}_B$ and $f_{g}$;
\end{algorithmic}
\end{algorithm}

We formalize the training algorithm of \method in Algorithm~\ref{alg:1}. In line 2, we select the poisoned nodes $\mathcal{V}_P$ with the top-$\Delta_P$ highest scores calculated by Eq.(\ref{eq:selection}). From lines 3 to 8, the lower-level optimization updates the surrogate model parameters $\theta$ based on Eq.(\ref{eq:lower-opt}), while the upper-level optimization updates the trigger generator parameters $\theta_g$ based on Eq.(\ref{eq:upper-opt}). From lines 9 to 12, the trained $f_g$ generates a trigger for each poisoned node $v_i\in\mathcal{V}_P$ and attaches it to $\mathcal{G}_B$.

\subsection{Time Complexity Analysis}\label{sec:time-complexity}
In \method, the time complexity mainly comes from the logic poisoning sample selection and the bi-level optimization of the logic poisoning trigger generator. Let $h$ denote the embedding dimension.
The cost of the logic poisoning node selection can be represented approximately as $O(M\bar{d}h|\mathcal{V}|)$, where $\bar{d}$ is the average degree of nodes and $M$ is the number of training iterations for the pre-trained GCN model, which is small. The cost of bi-level optimization consists of updating the weight of the surrogate GNN model in inner iterations and updating the logic poisoning trigger generator in outer iterations. The cost for updating the surrogate model is approximately $O(N\bar{d}h|\mathcal{V}_P|)$, where $N$ is the number of inner training iterations for the surrogate GNN model. For the trigger generator, the classification loss and prediction logic poisoning loss are computed with a cost of $O(2\bar{d}h|\mathcal{V}|)$. For the unnoticeable loss $\mathcal{L}_U$, its time complexity is $O(h\bar{d}|\mathcal{V}_P| \Delta_g)$. Hence, the overall time complexity of each iteration of bi-level optimization is $O(\bar{d}h(2|\mathcal{V}|+(\Delta_g+N)|\mathcal{V}_P|))$, which is linear to the size of the graph. Hence, {\method} can efficiently poison the inner prediction logic of target models for clean-label graph backdoor attacks.

\section{Experiments}

In this section, we conduct extensive experiments to answer the following research questions: 
\begin{itemize}[leftmargin=*,itemsep=2pt,topsep=3pt,parsep=0pt,partopsep=0pt]
    % \item \textbf{RQ1}:~How does \method perform compared to the state-of-the-art baselines under the clean-label setting?
    % \item \textbf{RQ2}:~How does \method generalize across diverse surrogate/target models, downstream GNN tasks, and graph types?
    % \item \textbf{RQ3}:~Can \method maintain high attack success rate against various defending strategies?
    % \item \textbf{RQ4}:~How effective are the proposed components of \method for clean-label graph backdoor attacks?
    \item \textbf{RQ1}:~How does \method perform against state-of-the-art baselines under the clean-label setting?
    \item \textbf{RQ2}:~How does \method generalize across diverse models, GNN tasks, and graph types?
    \item \textbf{RQ3}:~Can \method maintain high attack success rate against various defending strategies?
    \item \textbf{RQ4}:~How effective are the proposed components of \method for clean-label graph backdoor attacks?
\end{itemize}

\subsection{Experimental Settings}\label{sec:exp-setting}

\noindent \textbf{Datasets.}
We conduct extensive experiments on \textbf{Cora}~\cite{sen2008collective}, \textbf{Pubmed}~\cite{sen2008collective}, \textbf{Flickr}~\cite{zeng2019graphsaint}, and \textbf{Arxiv}~\cite{hu2020open} for node classification; 
\textbf{MUTAG}, \textbf{NCI1}, and \textbf{PROTEINS}~\cite{Morris2020} for graph classification; 
\textbf{Cora}, \textbf{CS}~\cite{hu2020open}, and \textbf{Physics}~\cite{hu2020open} for edge prediction.
We also include diverse heterophilous graphs, including \textbf{Squirrel} and \textbf{Chameleon}~\cite{luan2022revisiting}, \textbf{Penn} and \textbf{Genius} from~\cite{lim2021large}.
Details of the datasets are provided in Supplementary Section~D.1.

\noindent \textbf{Compared Methods.}
To highlight the superiority of our \method, we extensively compare \method with state-of-the-art graph backdoor attacks, including \textbf{GTA}~\cite{xi2021graph}, \textbf{EBA}~\cite{xu2021explainability}, \textbf{UGBA}~\cite{dai2023unnoticeable}, and~\textbf{DPGBA}~\cite{zhang2024rethinking} under clean-label settings.
These methods originally require altering the labels of poisoned nodes.
In our experiments, we extend them to the clean-label setting by only poisoning the labeled nodes of the target class. 
These extended baselines are denoted with the suffix \textbf{-C}.
We also compare our \method with \textbf{ERBA}~\cite{xu2022poster} and \textbf{ECGBA}~\cite{fan2024effective}, which are two latest clean-label graph backdoor attack methods.
We further extend comparison to latest graph backdoor attack methods for more graph tasks, including \textbf{SCLBA}~\cite{dai2025semantic}, \textbf{GCLBA}~\cite{meguro2024gradient}, \textbf{TRAP}~\cite{yang2022transferable} for graph classification;
and \textbf{SNTBA}~\cite{dai2024backdoor}, \textbf{PSO-LB} and \textbf{LB}~\cite{zheng2023link} for edge prediction. 
More details of the compared methods are provided in Supplementary Section~D.2.

\noindent \textbf{Evaluation Protocol.}\label{sec:evaluation_protocal}
In this work, we focus on an inductive learning setting, \textit{where attackers cannot access test samples during the graph poisoning.}
To achieve this, we randomly mask 50\% of the samples in the original graphs during training \method. 
Half of the masked samples are used as target samples for attack performance evaluation, while the other half is used for clean accuracy evaluation.
Specifically, we evaluate all backdoor attacks using the average attack success rate \textbf{(ASR)} on target samples attached with triggers and the clean accuracy of backdoored model on clean test samples. 
% Unless otherwise specified, a 2-layer GCN acts as the surrogate model for all attack evaluations.
A 2-layer GCN acts as the surrogate model for all evaluations.
To reduce randomness, we run experiments on each target model 5 times and report the average results.
More implementation details of \method are in Supplementary Section~D.3.

\subsection{Clean-Label Backdoor Performance}\label{sec:exp1}

\begin{table*}[t]
    \centering
    \caption{Average backdoor attack success rate and clean accuracy (ASR | clean accuracy (\%)). Note that the surrogate model deployed in {\method} is fixed as a 2-layer GCN.}
    \resizebox{\linewidth}{!}{
    \small
    \begin{tabular}{lcc*{7}{c}}
    \toprule
    \multirow{2}{*}{Dataset} & \multirow{2}{*}{\shortstack{Target\\Model}} & \multirow{2}{*}{\shortstack{Vanilla\\Acc.}} & \multirow{2}{*}{ERBA} & \multirow{2}{*}{ECGBA} & \multirow{2}{*}{EBA-C} & \multirow{2}{*}{GTA-C} & \multirow{2}{*}{UGBA-C} & \multirow{2}{*}{DPGBA-C} & \multirow{2}{*}{\method} \\
    & & & & & & & & & \\
    \midrule
    \multirow{3}{*}{{Cora}}
    & GCN & 83.78 & 18.22 | 80.77 & 34.77 | 79.48 & 29.13 | 74.17 & 32.45 | 80.45 & \cellcolor{gray!40}68.32 | 79.97 & 59.55 | 79.88 & \textbf{98.52 | 83.59} \\
    & GAT & 84.30 & 19.32 | 82.08 & 35.19 | 78.93 & 29.34 | 74.23 & 35.85 | 82.68 &  \cellcolor{gray!40}68.76 | 82.81 & 59.02 | 84.19 & \textbf{97.12 | 83.76} \\
    & GIN & 84.26 & 19.17 | 79.85 & 34.56 | 79.92 & 29.30 | 74.46 & 35.49 | 79.63 & \cellcolor{gray!40}67.75 | 83.04 & 60.03 | 81.56 & \textbf{98.97 | 83.81} \\
    \midrule
    \multirow{3}{*}{{Pubmed}}
    & GCN & 86.38 & 22.18 | 85.58 & 37.46 | 85.86 & 31.89 | 85.60 & 38.84 | 86.17 & \cellcolor{gray!40}71.24 | 85.31 & 64.19 | 85.41 & \textbf{96.75 | 86.03} \\
    & GAT & 86.51 & 22.24 | 85.91 & 41.54 | 85.61 & 30.13 | 85.46 & 42.14 | 85.78 & 66.07 | 86.33 & \cellcolor{gray!40}67.05 | 85.21 & \textbf{94.88 | 85.13} \\
    & GIN & 86.51 & 15.46 | 85.57 & 43.14 | 83.63 & 30.99 | 85.44 & 42.34 | 86.35 & \cellcolor{gray!40}68.69 | 85.81 & 66.17 | 86.22 & \textbf{99.04 | 86.21} \\
    \midrule
    \multirow{3}{*}{{Flickr}}
    & GCN & 46.21 & 0.00 | 45.75 & 39.47 | 45.68 & 32.47 | 45.68 & 48.12 | 44.96 & 66.49 | 43.99 & \cellcolor{gray!40}69.66 | 42.93 & \textbf{99.98 | 46.05} \\
    & GAT & 46.07 & 0.00 | 47.51 & 41.03 | 47.65 & 31.93 | 47.65 & 47.87 | 47.64 & 68.78 | 47.02 & \cellcolor{gray!40}70.39 | 46.31 & \textbf{99.72 | 44.91} \\
    & GIN & 46.22 & 0.00 | 45.62 & 41.71 | 41.64 & 32.07 | 41.64 & 48.04 | 45.03 & \cellcolor{gray!40}68.97 | 45.98 & 68.12 | 45.09 & \textbf{100.0 | 46.14} \\
    \midrule
    \multirow{3}{*}{{Arxiv}}
    & GCN & 66.58 & 0.01 | 66.14 & 25.56 | 66.16 & 28.64 | 66.23 & 37.16 | 66.29 & \cellcolor{gray!40}69.71 | 66.57 & 58.96 | 66.82 & \textbf{98.04 | 65.82} \\
    & GAT & 66.02 & 0.02 | 64.09 & 26.03 | 64.46 & 28.09 | 64.41 & 36.45 | 65.20 & \cellcolor{gray!40}71.65 | 65.10 & 59.13 | 65.25 & \textbf{98.43 | 65.40} \\
    & GIN & 66.73 & 0.02 | 66.07 & 26.72 | 62.01 & 27.35 | 66.08 & 34.32 | 65.87 & \cellcolor{gray!40}71.01 | 66.56& 60.50 | 66.73 & \textbf{97.62 | 66.83} \\
    \bottomrule
    \end{tabular}}
    \label{tab-exp1}
\end{table*}
To answer \textbf{RQ1}, we compare \method with six baselines across four datasets and three target GNNs under the clean-label settings outlined in Sec.~\ref{sec-threatmodel}. 
We select diverse GNNs as target models, including GCN, GAT, and GIN, and evaluate their vanilla accuracy on clean test nodes as a reference. 
We select four graphs with diverse characteristics, including Cora, Pubmed, Flickr, and Arxiv.
As described by the evaluation protocol in Sec.~\ref{sec:evaluation_protocal}, we report the results in Tab.~\ref{tab-exp1}, from which we observe:
\begin{itemize}[leftmargin=*,itemsep=2pt,topsep=2pt,parsep=0pt,partopsep=0pt]
    \item Across all datasets and models, \method consistently achieves the highest ASR, typically close to 100\%. It outperforms leading baselines such as UGBA-C and DPGBA-C, indicating that the clean-label setting is challenging for state-of-the-art methods, and \method poisons inner prediction logic of target models effectively for clean-label backdoor attacks. 
    \item Arxiv poses challenges due to its diverse classes and our fixed target-class setting. Despite requiring generalization to larger unseen graph parts, \method maintains superior performance as the ASR of baselines drops, demonstrating its scalability. Experiments on larger graphs are in Supplementary Section~A.2.
    \item High ASR of \method towards different target GNNs proves its transferability in backdooring various GNNs via logic poisoning. 
    % We further investigate how sampling strategies, where only a subset of nodes is involved during the aggregation, affect \method's performance in Appendix~\ref{sec:append-c3}.
    We further investigate how triggers crafted with one surrogate model affect more distinct target models in Sec.~\ref {sec:exp4-1}.
    \item  {\method} achieves comparable clean accuracy compared to vanilla GNN models, while other methods exhibit significantly larger clean accuracy drop. More experiments on the clean accuracy drop can be found in Supplementary Section~A.4.
\end{itemize}

% \begin{table}[t]
%     \centering
%     \small
%     \caption{ASR (\%) of {\method} in backdooring various target models on heterophilous graphs.}
%     \label{tab:append-hete}
%     \vspace{-0.75em} 
%     \begin{tabularx}{\linewidth}{lCCC}
%     \toprule
%     {Datasets} & {GCN} & {ACMGCN} & {LINKX} \\
%     \midrule
%     Squirrel  & 99.07 & 98.14 & 98.79 \\
%     Chameleon  & 99.03 & 98.53 & 98.43 \\
%     Penn  & 96.34 & 96.71 & 95.32 \\
%     Genius  & 98.17 & 97.92 & 93.71 \\
%     \bottomrule
%     \end{tabularx}
%     \vspace{-0.8em}
% \end{table}

\subsection{Generalization Ability of {\method}}\label{sec:exp4}

To answer \textbf{RQ2}, we evaluate the generalization ability of \method from three perspectives: diverse surrogate and target model architectures, various GNN downstream tasks, and different types of graphs.

\noindent \textbf{Flexibility to Surrogate and Target Model Architectures.}\label{sec:exp4-1}
We further evaluate the flexibility of \method by varying both the surrogate and target models across six backbones on Arxiv, as shown in Fig.~\ref{fig:surro_target}. 
From the figure, we observe: 
\textbf{(i)} Despite inherent differences between surrogate and target models, \method remains highly effective, demonstrating strong transferability rooted in poisoning the shared message-passing mechanism for most GNNs.
\textbf{(ii)} Sampling-based targets GraphSAGE and GraphSAINT (denoted as {SAGE} and {SAINT} in Fig.~\ref{fig:generalization}(a), respectively) appear marginally more robust against our \method. We attribute this to the sampling of a subset of nodes, which dilutes the trigger's impact. More analysis on how diverse sampling strategies affect \method's performance can be found in Supplementary Section~A.3.

\begin{figure*}[t]
    \centering
    \begin{subfigure}[t]{0.33\linewidth}
        \includegraphics[width=\linewidth,height=4.25cm,keepaspectratio]{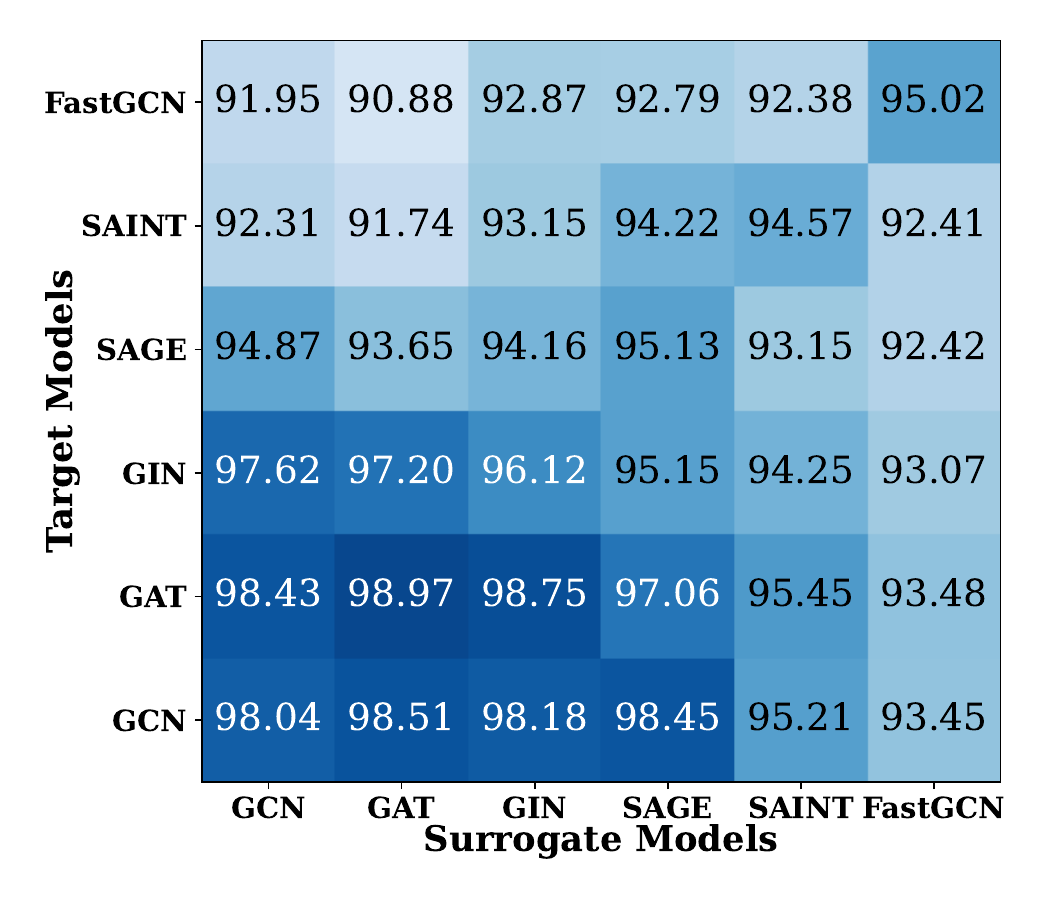}
        \caption{Surrogate \& Target Models}
        \label{fig:surro_target}
    \end{subfigure}
    \hfill
    \begin{subfigure}[t]{0.32\linewidth}
        \includegraphics[width=\linewidth,height=4.15cm,keepaspectratio]{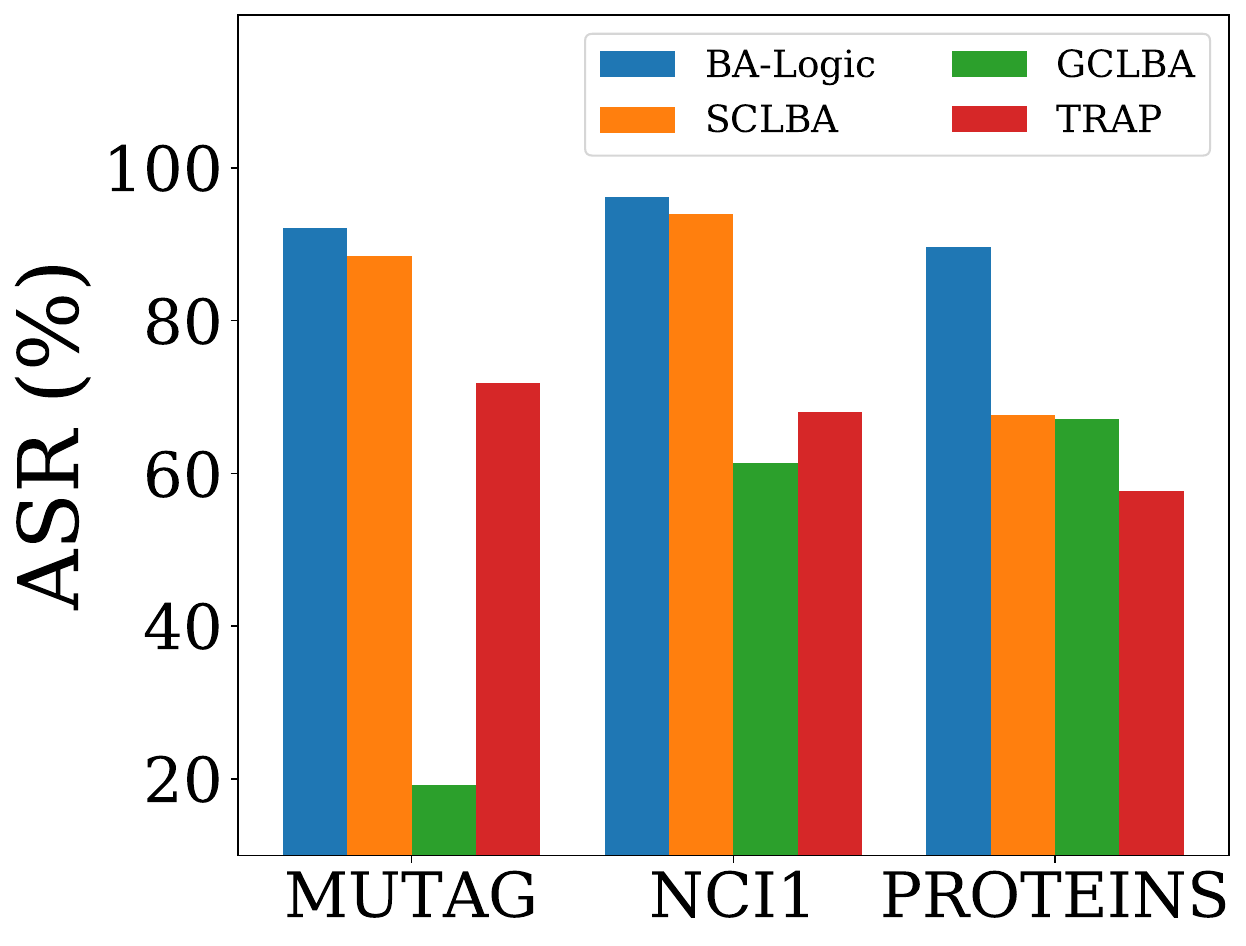}
        \caption{Graph Classification}
        \label{fig:graph_classification}
    \end{subfigure}
    \hfill
    \begin{subfigure}[t]{0.32\linewidth}
        \includegraphics[width=\linewidth,height=4.15cm,keepaspectratio]{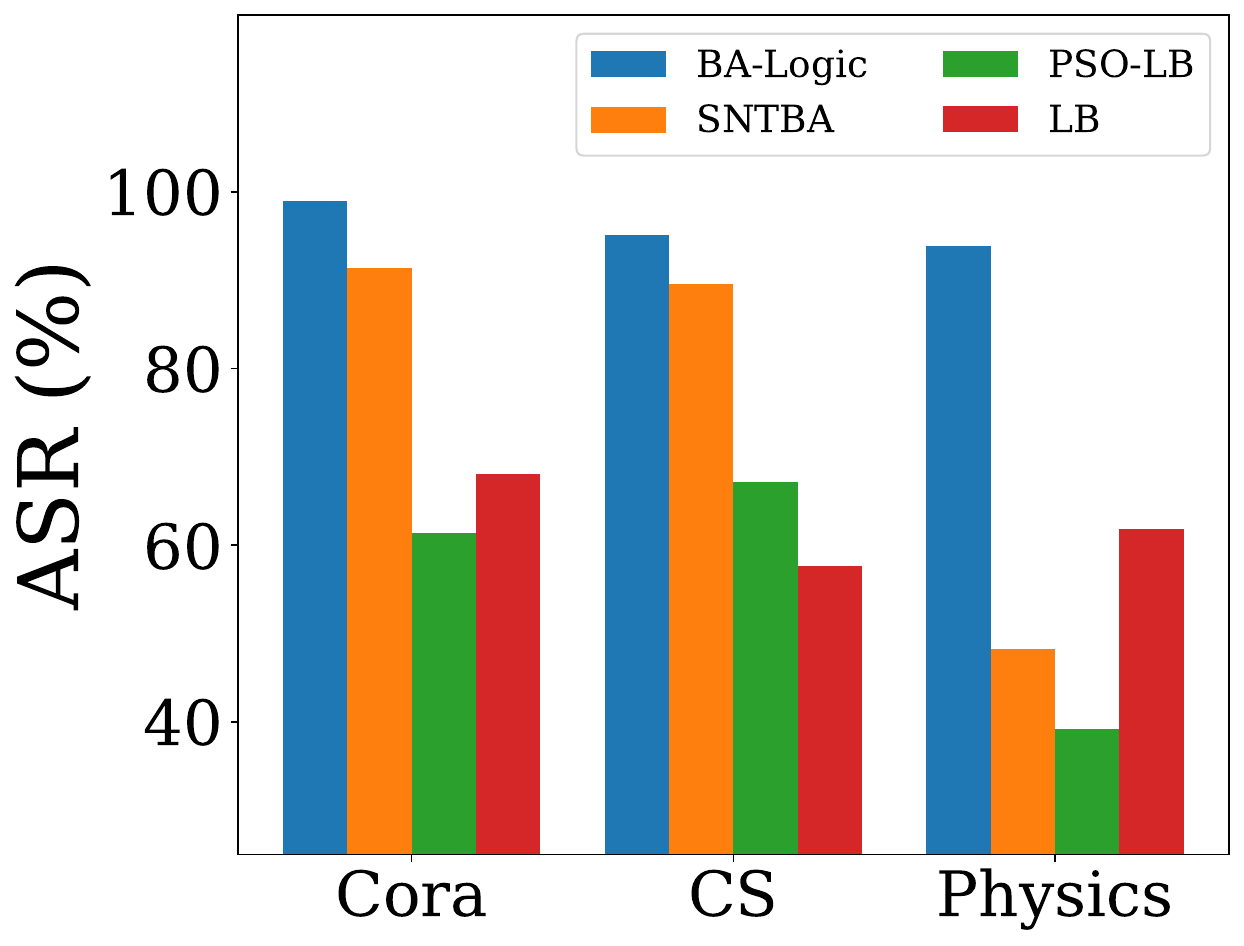}
        \caption{Edge Prediction}
        \label{fig:edge_prediction}
    \end{subfigure}
    \caption{Generalization ability of \method. (a) ASR (\%) with varying surrogate and target models on Arxiv. (b-c) ASR (\%) on graph classification and edge prediction.}
    \Description{A heatmap reports BA-Logic attack success rates across six surrogate and target GNN architectures on Arxiv. Two grouped bar charts compare BA-Logic with baselines on three graph-classification datasets and three edge-prediction datasets. BA-Logic achieves high success across model pairs and outperforms the compared methods on both tasks.}
    \label{fig:generalization}
\end{figure*}

\noindent \textbf{Generalization to More Tasks.}\label{sec:exp4-2}
As node classification, graph classification, and edge prediction share the message-passing mechanism of GNNs, our \method can be easily extended to graph classification and edge prediction tasks.
Details of the extensions can be found in Supplementary Section~A.1.
We compare our method with recent backdoor methods \textbf{SCLBA}~\cite{dai2025semantic}, \textbf{GCLBA}~\cite{meguro2024gradient}, and \textbf{TRAP}~\cite{yang2022transferable} for graph classification, and \textbf{SNTBA}~\cite{dai2024backdoor}, \textbf{PSO-LB}~\cite{zheng2023link}, and \textbf{LB}~\cite{zheng2023link} for edge prediction.
The ASR results are given in Fig.~\ref{fig:graph_classification} and \ref{fig:edge_prediction}, where we observe that \method achieves superior attack performance on both graph classification and edge prediction tasks. Complete results, including clean accuracy, are provided in Supplementary Section~A.1. This demonstrates the effectiveness of our \method in generalizing to other tasks.

\begin{table}[t]
    \centering
    \small
    \caption{ASR (\%) of {\method} in backdooring various target models on heterophilous graphs.}
    \label{tab:append-hete}
    \begin{tabular*}{\linewidth}{@{\extracolsep{\fill}}lccc@{}}
    \toprule
    {Datasets} & {GCN} & {ACMGCN} & {LINKX} \\
    \midrule
    Squirrel  & 99.07 & 98.14 & 98.79 \\
    Chameleon  & 99.03 & 98.53 & 98.43 \\
    Penn  & 96.34 & 96.71 & 95.32 \\
    Genius  & 98.17 & 97.92 & 93.71 \\
    \bottomrule
    \end{tabular*}
\end{table}
\noindent \textbf{Generalization to Heterophilous Graphs.}
We further evaluate \method on diverse heterophilous graphs,  with node classification as the downstream task. 
We target both the heterophily-specific GNN models \textbf{ACMGCN}~\cite{luan2022revisiting} and \textbf{LINKX}~\cite{lim2021large}, along with a standard 2-layer GCN. 
We report results in Tab.~\ref{tab:append-hete}.
From the table, we can observe that our logic poisoning method achieves high ASR across diverse graphs and backbones. 
This indicates our core idea, i.e., explicitly guiding the model's inner prediction logic to emphasize the injected triggers when predicting the poisoned nodes, maintains effectiveness consistently. 
More comparisons with baselines transferred from other domains are provided in Supplementary Sections~A.5 and~A.10.
More analysis of our generalization ability under challenging label and feature settings can also be found in Supplementary Section~G.

\subsection{Attacks Against Defense Strategies}\label{sec:exp2}
To answer \textbf{RQ3}, we evaluate the attack results against existing representative defense methods and defense strategies that we adaptively proposed for {\method}.

\noindent \textbf{Attacks Against Existing Defense Methods.}\label{sec:exp2-existing}
We first evaluate \method and baselines against representative defense methods, including \textbf{GCN-Prune}~\cite{dai2023unnoticeable}, \textbf{RobustGCN}~\cite{zhu2019robust}, \textbf{GNNGuard}~\cite{zhang2020gnnguard}, and \textbf{RIGBD}~\cite{zhang2025robustness}.
% Results on Arxiv are reported in the upper part of Tab.~\ref{tab-exp2}, while comprehensive evaluations and analysis are deferred to Appendix~\ref{sec:append-c2}.
Results on Arxiv are reported in Tab.~\ref{tab-exp2}, while comprehensive evaluations and analysis are deferred to Supplementary Section~A.6.
From the table, we observe that \method can effectively attack the robust GNN models. 
Compared with baselines, \method enhances ASR by at least 30 percentage points, showing the superiority of \method in attacking existing defense methods.
The results also highlight a gap, indicating that existing methods fail to defend against the logic poisoning introduced by \method.

\noindent \textbf{Attacks Against Adaptive Defense Methods.}\label{sec:exp2-adaptive}
We further propose four adaptive defenses tailored for logic poisoning. 
These defense methods are oriented from diverse perspectives, including explainability regularization \textbf{(ER)}, gradient masking \textbf{(GM)}, collaborative defense \textbf{(CD)}, and sampling and masking \textbf{(SAM)}. 
Their shared idea is that penalizing over-reliance on specific samples in predictions can mitigate the risk of poisoned logic. 
Results on Arxiv are reported in Tab.~\ref{tab-exp2}, while detailed analysis and evaluations are deferred to Supplementary Section~E.
From the table, we observe that \method consistently achieves the highest ASR, which exceeds 80\% across all adaptive defenses. 
The results highlight the resilience of our \method against adaptive defenses, and the superiority of our \method in poisoning inner logic for clean-label backdooring. 

\begin{table}[t]
    \centering
    \small
    \caption{Comparisons of ASR(\%) against various existing and adaptive defense models.}
    \label{tab-exp2}
    \resizebox{\linewidth}{!}{%
    \setlength{\tabcolsep}{2.5pt}
    \begin{tabular}{@{}llrrrr@{}}
        \toprule
        Category & {Defense} & {ECGBA} & {DPGBA-C} & {UGBA-C} & {\method} \\
        \midrule
        \multirow{4}{*}{\shortstack{Existing\\Defenses}}
            & GCN-Prune & 13.57 & 17.43 & 62.31 & \textbf{96.75} \\
            & RobustGCN & 14.24 & 29.65 & 44.46 & \textbf{97.03} \\
            & GNNGuard  & \phantom{0}0.51 & 43.77 & 40.08 & \textbf{95.37} \\
            & RIGBD     & \phantom{0}0.01 & \phantom{0}0.01 & \phantom{0}0.19 & \textbf{93.23} \\
        \midrule
        \multirow{4}{*}{\shortstack{Adaptive\\Defenses}}
            & ER  & 19.84 & 51.73 & 44.39 & \textbf{80.06} \\
            & GM  & 17.92 & 46.25 & 49.87 & \textbf{92.36} \\
            & CD  & 18.63 & 48.37 & 42.51 & \textbf{85.38} \\
            & SAM & 18.21 & 47.06 & 40.94 & \textbf{81.09} \\
        \bottomrule
    \end{tabular}}
\end{table}

\subsection{Ablation Studies}\label{sec:exp5}

\begin{figure*}[t]
    \centering
    \begin{subfigure}[t]{0.33\linewidth}
        \includegraphics[width=\linewidth]{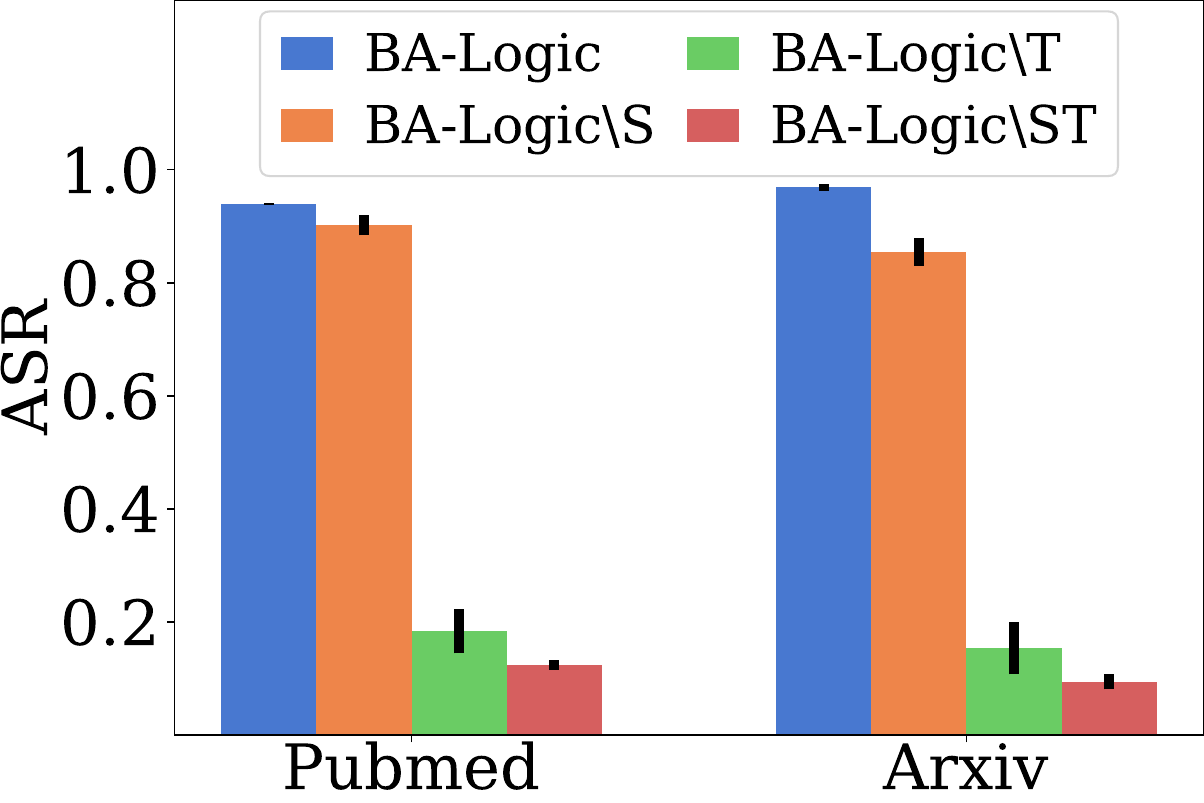}
        \caption{GCN as target}
    \end{subfigure}
    \hfill
    \begin{subfigure}[t]{0.33\linewidth}
        \includegraphics[width=\linewidth]{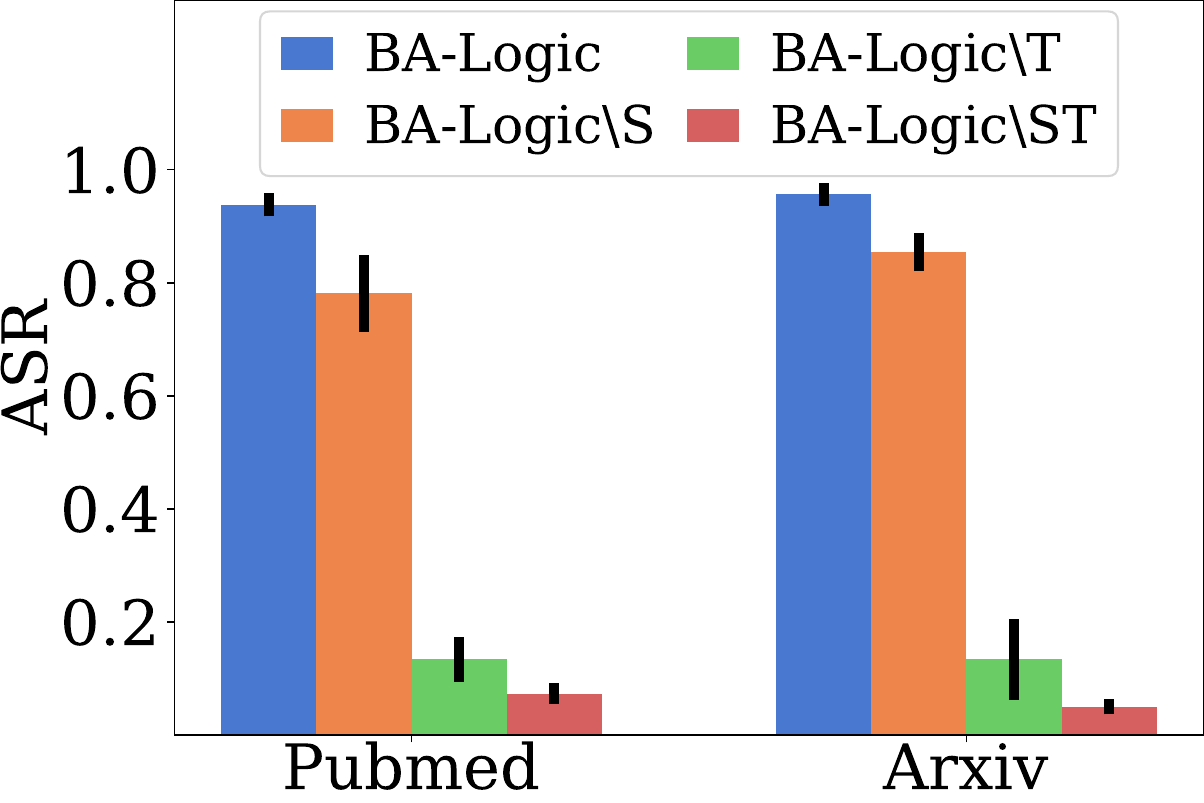}
        \caption{GIN as target}
    \end{subfigure}
    \hfill
    \begin{subfigure}[t]{0.33\linewidth}
        \includegraphics[width=\linewidth]{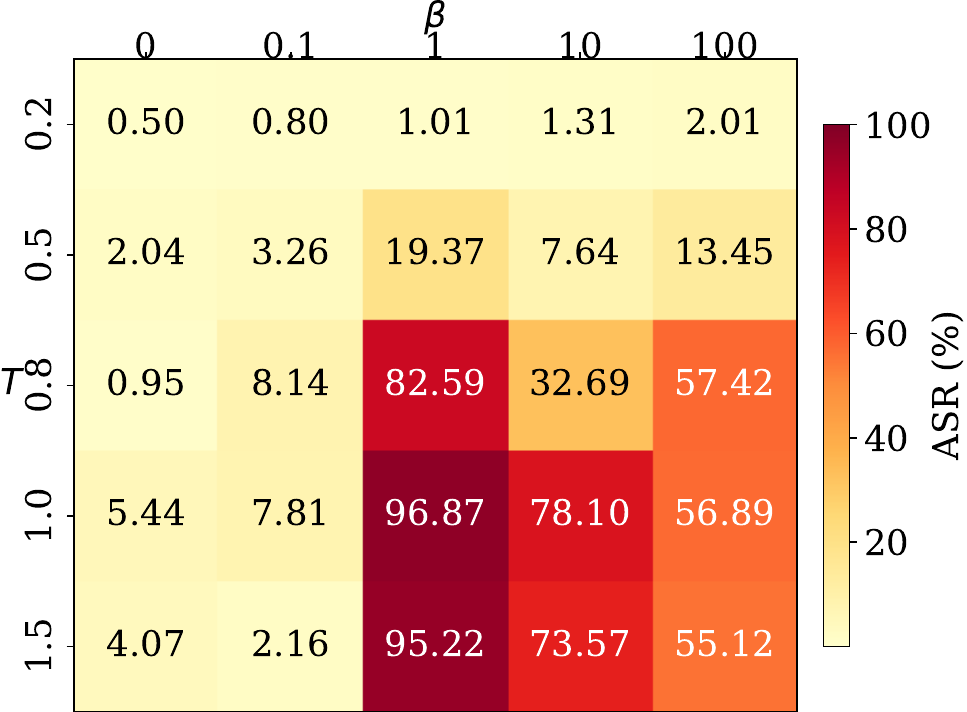}
        \caption{Sensitivity of $T$ and $\beta$}\label{fig:new-sa}
    \end{subfigure}
    \caption{(a-b) Ablation studies on GCN and GIN. (c) Hyperparameter sensitivity analysis on Arxiv.}
    \Description{Two grouped bar charts compare BA-Logic with variants that remove node selection, logic-poisoning loss, or both, on Pubmed and Arxiv using GCN and GIN targets. Removing the logic-poisoning loss markedly reduces attack success. A heatmap shows attack success on Arxiv as the margin T and loss weight beta vary.}
    \label{fig:6}
\end{figure*}

To answer \textbf{RQ4}, we conduct ablation studies to explore the effectiveness of the poisoned node selector and the logic-poisoning trigger generator. 
To demonstrate the effectiveness of poisoned node selector, we randomly select poisoned nodes from the training graph and obtain a variant named {\method}$\backslash$S.
To demonstrate the benefit of logic poisoning trigger generator, we remove the inner logic loss in Eq.~(\ref{eq:grad_loss}). In such a case, the \method degrades to a simplified variant named {\method}$\backslash$T. 
We also implement {\method}$\backslash$ST, a variant removing both modules.
The ASR and standard deviations on Pubmed and Arxiv are shown in Fig.~\ref{fig:6}, where we observe that:
\begin{itemize}[leftmargin=*,itemsep=2pt,topsep=2pt,parsep=0pt,partopsep=0pt]
    \item Compared to {\method}$\backslash$S, \method achieves better attack performance. 
    The variance of ASR of \method is significantly lower than that of {\method}$\backslash$S. It indicates that our selector consistently identifies diverse and influential nodes for logic poisoning.
    \item \method consistently outperforms both {\method}$\backslash$T and {\method}$\backslash$ST by a large margin. It highlights that our logic poisoning loss effectively guides the trigger generator to produce triggers that poison the target GNN's inner logic across diverse test nodes.
\end{itemize}
More analysis on the contribution of each module of our \method can be found in Supplementary Section~A.7.

\subsection{Hyperparameter Analysis}\label{sec:exp-hyper}
% \enyan{Generated by dai-write agent, to be reviewed by Yuxiang.}
We investigate the sensitivity of \method to two key hyperparameters: the margin $T$ in Eq.~(\ref{eq:grad_loss}) that controls the expected importance gap between trigger and clean nodes, and the weight $\beta$ in Eq.~(\ref{eq:bilevel}) that balances the logic poisoning loss.
We conduct parameter sweeps on four datasets and report results on Arxiv in Fig.~\ref{fig:new-sa}, with complete results in Supplementary Section~A.7.
We observe that \textbf{(i)} graphs with higher average node degree (e.g., Arxiv) require larger $T$ to overcome the stronger influence of clean neighbors on the prediction logic, as more clean nodes dilute the trigger's importance scores;
\textbf{(ii)} excessively large $T$ or $\beta$ degrades ASR, as extreme values hinder the bi-level optimization of \method by causing gradient instability in the trigger generator;
\textbf{(iii)} \method maintains robust performance across a wide range of $T$ and $\beta$ values, indicating that our method is not sensitive to hyperparameter choices in practice.

\section{Related Works}\label{append:related}
% \enyan{Add the related work section.}\zyx{updated}
\noindent \textbf{Graph Neural Networks.} \label{sec:related-a}
% \subsection{Graph Neural Networks} \label{sec:related-a}
Graph Neural Networks (GNNs) come into the spotlight due to their remarkable ability to model graph-structured data~\cite{chen2020simple, hamilton2017inductive, gcn, gasteiger2018predict, gat, wu2019simplifying}. 
Recently, many GNN models have been proposed to further improve GNN performance~\cite{dai2024comprehensive}. 
There are also works that tailored to address the fairness~\cite{dai2021say}, robustness~\cite{wang2023robust, dai2022towards}, and explainability~\cite{pope2019explainability, dai2021towards, ma2026explanations} challenge of GNNs.
GNN models for heterophilous graphs are also designed~\cite{luan2022revisiting,zhang2019heterogeneous,lim2021large}. 
While the representational powers of GNN models have been well studied, their performance under diverse backdoor attacks has remained an open question~\cite{zhang2021backdoor}. 
Our analysis reveals that GNNs' vulnerability to backdoor depends on whether their inner prediction logic is poisoned by the trigger pattern.
Based on this, we propose a novel method that explicitly guides the model’s inner prediction logic to emphasize the injected triggers when predicting the target nodes.

% \subsection{Graph Backdoor Attacks} \label{sec:related-b}
\noindent \textbf{Graph Backdoor Attacks.} \label{sec:related-b}
Exploring backdoor attacks on graphs has aroused increasing interest among the graph learning community~\cite{zhang2021backdoor, xi2021graph, dai2023unnoticeable, xu2022poster, wang2024multi, zhang2024rethinking}. 
% SBA~\cite{zhang2021backdoor} is a seminal work on graph backdoor attacks that adopts randomly generated graphs as triggers.
Recent work achieves unnoticeable backdoors via similarity~\cite{dai2023unnoticeable} and distribution~\cite{zhang2024rethinking} constraints in a dirty-label setting. 
For the clean-label setting, some initial efforts have been proposed~\cite{xu2022poster, xing2023clean, chen2024agsoa, dai2025semantic, xia2025clean}. 
Besides injecting subgraphs as triggers, a body of research has been proposed to conduct backdoor attacks by changing the most representative node features~\cite{xu2021explainability, xing2023clean, wang2024explanatory} or manipulating existing node and edges as triggers for graph classification and link prediction~\cite{yang2022transferable, meguro2024gradient, dai2025semantic, dai2024backdoor, zheng2023link}.
However, modifying a graph's current features or structure is often impractical in real-world scenarios~\cite{alothali2018detecting}.
Distinct from prior work, we study a novel problem of poisoning the prediction logic of target models for clean-label graph backdooring, and propose \method that optimizes the importance scores of triggers to exceed those of the clean neighbors of attached nodes for logic poisoning.

% \subsection{Explaining the Prediction Logic of GNNs}\label{sec:related-c}
\noindent \textbf{Explaining the Prediction Logic of GNNs.}\label{sec:related-c}
To enhance the trustworthiness of GNN's predictions, researchers have developed various explanation methods for GNNs~\cite{huang2022graphlime,schnake2021higher,luo2020parameterized}, such as GNNExplainer~\cite{ying2019gnnexplainer} that leverages mutual information to identify the most relevant subgraph, Grad-CAM variants tailored to GNNs~\cite{pope2019explainability,selvaraju2017grad}, and $\pi$-GNN~\cite{yin2023train} that pre-trains on synthetic graphs with ground-truth explanations.
Explainability research makes the prediction logic of GNNs traceable, which in turn helps researchers understand GNN behavior~\cite{zhang2022trustworthy,dai2024comprehensive,tang2023explainable,wang2024trustguard}.
In this work, we leverage GNN explainability techniques to analyze and manipulate the prediction logic of GNNs for effective clean-label graph backdoor attacks.
\method instead optimizes node importance for logic poisoning.

\section{Conclusion and Future Works}\label{sec:conclusion}
We study logic poisoning in GNNs for clean-label graph backdoor attacks. Motivated by empirical and theoretical analysis, \method learns logic-poisoning triggers through bi-level optimization. Experiments across datasets, tasks, architectures, and defenses demonstrate its effectiveness. Future work includes extending \method to recommendation systems and developing defenses against logic poisoning that preserve clean accuracy.

\begin{acks}
\setlength{\emergencystretch}{3em}
This material is based upon work supported by, or in part by, Prevention and Control of Emerging and Major Infectious Diseases-National Science and Technology Major Project (Grant No. 2026ZD01910200), the National Natural Science Foundation of China (NSFC) under Grant No. 62506316, the Guangdong Provincial Program under Grant No. 2025D03J0015, and the Guangdong Basic and Applied Basic Research Foundation under Grant No. 2025A1515140107. The findings in this paper do not necessarily reflect the views of the funding agencies.
\par
\end{acks}

% Only references and ethical considerations may extend beyond page 9.

\section*{Ethical Considerations}
In this work, we investigate potential security vulnerabilities in Graph Neural Networks (GNNs) by developing an inner prediction logic poisoning clean-label graph backdoor attack without modifying existing labels.
We recognize that the study of logic poisoning backdoor attacks may raise concerns regarding potential risks. 
The proposed attack could be misused to compromise GNN-based systems when an adversary can inject triggers into their training graphs.
However, our primary goal is to advance the understanding of such vulnerabilities to foster the development of more robust and trustworthy GNN models.
In practice, this risk can be mitigated by restricting write access to training graphs and applying data-cleaning or robust-training defenses before deployment.
We believe this work will encourage further research into understanding clean-label graph backdoor attacks and contribute to improving the safety and reliability of graph learning methods.

\bibliographystyle{ACM-Reference-Format}
\balance
\bibliography{ref}
\end{document}